\documentclass{article}

\usepackage{mathtools}
\DeclarePairedDelimiter\ceil{\lceil}{\rceil}
\DeclarePairedDelimiter\floor{\lfloor}{\rfloor}

\usepackage[preprint]{neurips_2024}

\usepackage{amsthm}
\newtheorem{definition}{Definition}[section]
\usepackage[normalem]{ulem}
\usepackage{mathtools}
\usepackage[linesnumbered,ruled,vlined]{algorithm2e}
\usepackage{array}
\usepackage{graphicx}
\usepackage{subcaption}

\usepackage[utf8]{inputenc} 
\usepackage[T1]{fontenc}    
\usepackage{hyperref}       
\usepackage{url}            
\usepackage{booktabs}       
\usepackage{amsfonts}       
\usepackage{nicefrac}       
\usepackage{microtype}      
\usepackage{xcolor}         
\usepackage{amsthm}
\newtheorem{example}{Example}
\usepackage{bbm}
\usepackage[ruled,vlined,linesnumbered]{algorithm2e}
\usepackage{amsmath,amsthm}
\usepackage{cleveref}
\newcommand{\probP}{\text{I\kern-0.15em P}}
\usepackage{dsfont}
\newtheorem{theorem}{Theorem}

\newtheorem{remark}{Remark}
\newtheorem{lemma}{Lemma}

\usepackage{thm-restate}
\usepackage[linesnumbered,ruled,vlined]{algorithm2e}
\usepackage{array}
\usepackage{graphicx}
\usepackage{subcaption}
\usepackage[utf8]{inputenc} 
\usepackage[T1]{fontenc}    
\usepackage{hyperref}       
\usepackage{url}            
\usepackage{booktabs}       
\usepackage{amsfonts}       
\usepackage{nicefrac}       
\usepackage{microtype}      
\usepackage{xcolor}         
\usepackage{amsthm}
\title{Bandit Learning in Matching Markets: Utilitarian and Rawlsian Perspectives}

\author{%
Hadi Hosseini\\  
Penn State University, USA\\ 
\texttt{hadi@psu.edu}\\
\And
Duohan Zhang\\
Penn State University, USA\\ 
\texttt{dqz5235@psu.edu} \\
}

\begin{document}

\maketitle

\begin{abstract}
  Two-sided matching markets have demonstrated significant impact in many real-world applications, including school choice, medical residency placement, electric vehicle charging, ride sharing, and recommender systems.
However, traditional models often assume that preferences are known, which is not always the case in modern markets, where preferences are unknown and must be learned.
For example, a company may not know its preference over all job applicants a priori in online markets.
Recent research has modeled matching markets as multi-armed bandit (MAB) problem and primarily focused on optimizing matching for one side of the market, while often resulting in a pessimal solution for the other side.
In this paper, we adopt a welfarist approach for both sides of the market, focusing on two metrics: (1) Utilitarian welfare and (2) Rawlsian welfare, while maintaining market stability.
For these metrics, we propose algorithms based on epoch Explore-Then-Commit (ETC) and analyze their regret bounds.
Finally, we conduct simulated experiments to evaluate both welfare and market stability.
\end{abstract}

\section{Introduction}

Two-sided matching markets provide a foundational framework for addressing problems involving matching two disjoint sets---commonly referred to as agents and arms---based on their preferences over the other side of the market.
These markets have had significant impact on a wide range of application, ranging from online digital markets such as recommender systems~\citep{eskandanian2020using}, electric vehicle charging~\citep{gerding2013two}, and ride sharing~\citep{Banerjee2019}, 
to traditional market design problems including medical residency matching \citep{R84evolution}, school choice \citep{APR05new,APR+05boston}, and labor market \citep{RP99redesign}.


The goal is often to find a \textit{stable} matching between the two sets (e.g. freelancers and job requesters) such that no pair prefers each other to the matched partner prescribed within the market, ensuring long-term success of the markets \citep{roth2002economist} by eliminating the incentives for participants to `scramble' to engage in secondary markets \citep{kojima2013matching}.

In a vast majority of applications, particularly those prevalent in digital marketplaces, preferences may not be readily available and must be learned.
Thus, a recent line of research models matching markets as \textit{multi-armed bandit} problems---where preferences of one or both sides of the market are unknown---and aims at learning preferences through sampling \citep{das2005two},  analyzing the sample complexity of finding stable matchings \citep{hosseini2024puttinggaleshapley}, or achieving optimal welfare for one side of the market \citep{pmlr-v108-liu20c, pmlr-v130-sankararaman21a, pmlr-v139-basu21a, NEURIPS2022_615ce9f0, NEURIPS2022_3e36cbff, doi:10.1137/1.9781611977554.ch55, wang2022bandit}.
These works almost exclusively rely on the seminal \textit{deferred acceptance} (DA) algorithm, which favors one side of the market without any welfare consideration for the other side \citep{gale1962college}. In fact, any optimal matching for one side (agent-optimal) is necessarily pessimal for the other side (arm-pessimal) \citep{mcvitie1971stable}, which could render the solution unfair.

In this paper, we consider two rather orthogonal approaches for measuring the welfare of the participants in both sides of the market.
In particular, the \textit{utilitarian} objective  measures the welfare of all participants (sum of utilities of both agents and arms), and the \textit{Rawlsian} objective that measures the welfare of the market according to the utility of its worst-off member. 
The goal is to find matchings that maximize the utilitarian welfare (utilitarian-optimal) or maximize the Rawlsian welfare (maximin) among all stable solutions.\footnote{In the matching literature, the former is often referred to as egalitarian-optimal while the latter is called regret-optimal (see, e.g. \cite{irving1987efficient,gusfield1987three}).
We change the terminology here to i) avoid confusion with `regret' of bandit learning and ii) to emphasize that preferences are prescribed by cardinal utilities (rather than the traditional ordinal rankings).}
The following example illustrates how in a small market, agent-optimal, arm-optimal, utilitarian-optimal, and Rawlsian maximin may yield rather different (stable) matchings.

\begin{example}
Consider four agents $\{a_1,a_2,a_3,a_4\}$ and four arms $\{ b_1,b_2,b_3,b_4\}$. To help understand the intuition, the strict linear order preferences are shown. The associated utilities are shown in parentheses. 


\begin{minipage}{0.4\linewidth}
\footnotesize
\begin{align*}
a_1 &: \underline{b_1} \succ b_2^* \succ b_3^{\dag} \succ \uwave{b_4} \\
a_2 &: \underline{b_2} \succ b_3^* \succ b_4^{\dag} \succ \uwave{b_1} \\
a_3 &: \underline{b_3} \succ b_4^* \succ b_1^{\dag} \succ \uwave{b_2} \\
a_4 &: \underline{b_4} \succ b_1^* \succ b_2^{\dag} \succ \uwave{b_3} 
\end{align*}
\end{minipage}
\hfill
\begin{minipage}{0.4\linewidth}
\footnotesize
\begin{align*}
b_1 &: \uwave{a_2 (3.5)} \succ a_3 (3.0)^{\dag} \succ a_4 (1.8)^* \succ \underline{a_1 (0.5)} \\
b_2 &: \uwave{a_3 (3.5)} \succ a_4 (3.0)^{\dag} \succ a_1 (1.7)^* \succ \underline{a_2 (0.5)} \\
b_3 &: \uwave{a_4 (3.5)} \succ a_1 (3.0)^{\dag} \succ a_2 (1.8)^* \succ \underline{a_3 (0.5)} \\
b_4 &: \uwave{a_1 (3.5)} \succ a_2 (3.0)^{\dag} \succ a_3 (1.7)^* \succ \underline{a_4 (0.5)}
\end{align*}
\end{minipage}

Agents' utilities over arms are omitted above.
Each agent's utility is $3.5, 2.5, 1.5, 0.5$ for arms (according to its preference). 
Note that preferences may be different, but the associated utilities are the same.

In this example, four stable matchings are shown above.
The underlined matching is the agent-optimal stable matching, the matching denoted by $^*$ is the maximin stable matching, the matching denoted by $^{\dag}$ is the utilitarian-optimal stable matching, and the matching with a wavy underline is the arm-optimal stable matching. 
\end{example}

\paragraph{\textbf{Our contributions.}}
We study both utilitarian-optimal and maximin objectives in stable matching markets. 
Our model generalizes the previous works by assuming that preferences of both sides are unknown and must be learned.
We propose two algorithms based 
 on variant of Explore-then-Commit algorithm where in each epoch agents explore arms uniformly in a round-robin way and then commit to the desired arm based on estimated utilities.
We show that our utilitarian epoch ETC algorithm achieves the regret bound of $\tilde{O}(N^2\log(T))$ (\Cref{thm:ega_regret}), while its maximin counterpart has a regret bound of $\tilde{O}(N\log(T))$ as shown in \Cref{thm:mm_regret}.
We proposed two techniques that measure the amount of error tolerable in finding the optimal matchings: the within-side minimum preference gap and the cross-side minimum preference gap. 
We utilize the two preference gaps in analyzing the corresponding regret bounds.
%
Finally, we empirically validate the regret and the stability of the proposed algorithms on randomly generated instances.

\subsection{Related Work}

The problem of a two-sided matching market has been widely studied in the literature during the past few decades.
The deferred-acceptance (DA) algorithm is a well-known algorithm that guarantees a stable solution through iterative proposals and rejections \citep{gale1962college}.
The resulting stable matching is optimal for the proposing side but pessimal for the accepting side, in which each member in the proposing side has the best partner and each member in the accepting side has the worst partner among all stable matchings.
Several fairness notions have been proposed such as utilitarian-optimal stable matching \citep{irving1987efficient}, maximin stable matching \citep{gusfield1987three}, sex-equal stable matching \citep{kato1993complexity}.
The utilitarian-optimal stable matching and maximin stable matching can be found in polynomial time, while finding a sex-equal stable matching is NP-hard \citep{kato1993complexity}.

\citet{das2005two} first formalized the matching market problem in a bandit setting, where each side shares the same preference profiles. 
\citet{pmlr-v108-liu20c} studied a variant of the problem in which agents have unknown preferences over arms, while the arms have known preferences over the agents.
They proposed two metrics: agent-optimal stable regret, and agent-pessimal stable regret, which compare the solution with either the agent-optimal stable partner or the arm-optimal stable partner.
Some follow-up work~\citep{pmlr-v130-sankararaman21a, pmlr-v139-basu21a, NEURIPS2022_615ce9f0, wang2024optimal} focused on special preference profiles where a unique stable matching exists.
Recently, \citet{doi:10.1137/1.9781611977554.ch55, NEURIPS2022_3e36cbff} proposed decentralized algorithms to achieve the agent-optimal stable regret bound $O(\frac{KlogT}{\Delta^2})$ when the market is decentralized, where $K$ is the number of arms, $T$ is the horizon and $\Delta$ is the within-side minimum preference gap. 
\citet{hosseini2024puttinggaleshapley} proposed an arm-proposing DA type algorithm and analyzed the sample complexity to find a stable matching under the probably approximately correct setting.

Some other works studied alternative models of matching markets in learning setting.
\citet{zhangdecentralized} studied the bandit problem in matching markets, where both sides of participants have unknown preferences. 
\citet{wang2022bandit, kong2024improved} studied many-to-one matching markets, where one arm can be matched to multiple agents.
\citet{jagadeesan2021learning, cen2022regret} studied the bandit learning problem in matching markets that allow monetary transfer.
\citet{ravindranath2021deep} studied the mechanism design problem in matching markets through deep learning.

\section{Preliminary}

\paragraph{\textbf{Problem setup.}}
A two-sided matching problem is composed of $N$ agents $\mathcal{N} = \{a_1, a_2, \ldots, a_N\}$ on one side, and $K$ arms $ \mathcal{K} = \{ b_1, b_2, \ldots, b_K\}$ on the other side. 
For simplicity, we assume $N = K$ to ensure all agents and all arms are matched.\footnote{In \Cref{rem:unequal}, we discuss how this assumptions is relaxed to unequal sides.}
The preference of an agent $a_i$, denoted by $\succ_{a_i}$, is a strict total ordering over the arms. 
Each agent $a_i$ has utility $\mu_{a_i,b_j}$ over arm $b_j$.
We say an agent $a_i$ prefers arm $b_j$ to $b_k$, i.e. $b_j \succ_{a_i} b_k$,  if and only if $\mu_{a_i,b_j} > \mu_{a_i,b_k}$.
Similarly, the preference of an arm $b_j$ over agents is denoted by $\succ_{b_j}$.
Each arm $b_j$ has utility $\mu_{b_j,a_i}$ over agent $a_i$, and arm $b_j$ prefers agent $a_i$ to $a_k$, i.e. $a_i \succ_{b_j} a_k$, if and only if $\mu_{b_j,a_i} > \mu_{b_j,a_k}$.
We assume that all utilities are non-negative and bounded by a constant.
%
We use $\mu$ to indicate the \textit{utility (preference) profile} of all agents on both sides.

%
%

\paragraph{\textbf{Stable matching.}} A matching is a mapping $m: \mathcal{N} \cup \mathcal{K} \to \mathcal{N} \cup \mathcal{K}$ such that $m(a_i) \in \mathcal{K}$ for all $i \in [N]$, and $m(b_j) \in \mathcal{N}$ for all $j \in [K]$, $m(a_i) = b_j$ if and only if $m(b_j) = a_i$. 
Given a matching $m$, an agent-arm pair $(a_i,b_j)$ is called a blocking pair if they prefer each other over their assigned partners, i.e. $b_j \succ_{a_i} m(a_i)$ and $a_i \succ_{b_j} m(b_j)$.
 A matching is stable if there is no blocking pair.
 Denote $\mathcal{M}$ as the set of all stable matchings.
 It is well known that $\mathcal{M}$ has at least one matching \citep{gale1962college} and could be exponential in size \citep{knuth1976mariages}. 

The \emph{Deferred Acceptance (DA) algorithm}~\citep{gale1962college} efficiently identifies a stable matching through the following process: participants on the proposing side make proposals based on their preferences to those on the receiving side. The receiving side temporarily accepts the most preferred proposals and rejects the others. This process repeats until all participants on the proposing side either have their proposals accepted or have exhausted their list of preferences and remain unmatched.
The matching returned by DA is optimal for the proposing side \citep{gale1962college} and pessimal for the receiving side \citep{mcvitie1971stable}.
Thus, depending on which side make proposals, the DA algorithm is either \textit{agent-optimal} (arm-pessimal) or \textit{arm-optimal} (agent-pessimal).



\paragraph{\textbf{Rewards and Bandits.}} 
%
The preferences of participants on both sides of the market (i.e. agents and arms) are unknown.
Denote by $T$ the time horizon. In each time step $t\leq T$, agent $a_i$ pulls an arm, denoted by $A_i(t)$. If agent $a_i$ pulls arm $b_j$, the agent $a_i$ receives a stochastic reward drawn from a 1-subgaussian distribution\footnote{A random variable $X$ is $d$-subgaussian if its tail probability satisfies $P(|X|>t) \le 2 \exp(-\frac{t^2}{2d^2})$ for all $t \ge 0$.} with mean value $\mu_{a_i,b_j}$. 
At the same time, arm $b_j$ gets a stochastic reward drawn from 1-subgaussian distribution with mean value $\mu_{b_j,a_i}$.
We use $A_j^{-1}(t)$ to denote the agent that pulls arm $b_j$ at time $t$.
We denote the corresponding sample average of agent $a_i$ on arm $b_j$ as $\hat{\mu}_{a_i,b_j}$, and the sample average of arm $b_j$ on agent $a_i$ as $\hat{\mu}_{b_j,a_i}$. 
If multiple agents pull the same arm at time $t$, conflicts arise and all participants fail and get $0$ rewards.

\section{Learning Utilitarian-Optimal Stable Matching}
In this section, we first define the utilitarian welfare in stable matching markets and formally describe the algorithm for computing the utilitarian-optimal stable matching when preferences are known (\Cref{sec:utilOpt}).
In \Cref{sec:epochETC-util}, we develop an algorithm based on an epoch ETC algorithm and analyze its regret bound.

\subsection{Utilitarian-Optimal Stable Matching} \label{sec:utilOpt}

The utilitarian welfare of a matching is the sum of utilities/rewards of all agents and arms. 

\begin{definition}[Utilitarian-Optimal Stable Matching]
    The utilitarian welfare of a stable matching $m$ is the sum of rewards of all agents and arms, i.e. 
    \begin{equation*}
    R(m) = \sum_{i = 1}^{N}\mu_{a_i, m(a_i)} + \sum_{j= 1}^{N}\mu_{b_j, m(b_j)}
    \end{equation*}
    A \textbf{utilitarian-optimal stable matching} $m^{*}$ is a stable matching that has the optimal utilitarian welfare, i.e. 
 $
     m^* = \arg\max_{m \in \mathcal{M}} R(m)
 $.   
\end{definition}
Similarly, we define $\hat{R}(m)$ as the estimated utilitarian welfare based on sample average, i.e.
\begin{equation*}
    \hat{R}(m) = \sum_{i = 1}^{N}\hat{\mu}_{a_i, m(a_i)} + \sum_{j= 1}^{N}\hat{\mu}_{b_j, m(b_j)}.
\end{equation*}
Denote $\hat{m}$ as the stable matching with respect to $\hat{\mu}$ that has the largest estimated utilitarian welfare.





\begin{algorithm}[t]\small
    \SetKwInOut{Input}{Input}
    \SetKwInOut{Output}{Output}
    \Input{Agents' utilities and arms' utilities.}
    \Output{A utilitarian-optimal stable matching.}
    Find the agent-optimal stable matching $m^a$ and the arm-optimal stable matching $m^b$ through DA algorithms;\\
    Starting from $m^a$, break existing matchings, track rotations, and stop until finding $m^b$;\\
    Construct a directed graph with each node corresponding to a rotation, edges denote predecessors;\\
    Find the sparse graph;\\
    Construct an $s-t$ flow graph and find the min-cut;
    \caption{Utilitarian-Optimal Algorithm} \label{alg:weighted_egalitarian}
\end{algorithm}

When preferences are known, \citet{irving1987efficient} proposed an algorithm that computes a utilitarian-optimal stable matching in polynomial time when preferences are ordinal.
Below, we briefly describe the steps of the algorithm, and refer the reader to the appendix for detailed formalism and explanation of the techniques.

\paragraph{\textbf{Computing a Utilitarian-Optimal Stable Matching.}}
\Cref{alg:weighted_egalitarian} utilizes the generalization of \citet{irving1987efficient}'s technique to cardinal preferences. 
The algorithm proceed in the following steps: 
i) finding agent-optimal and arm-optimal stable matchings by performing agent-proposing and arm-proposing DA, respectively,
ii) finding all rotations through the break-matching operation,
iii) constructing a rotation graph with each node representing a rotation, each edge from a node to another denoting predecessor relation, and weights assigned to nodes based on utilities,
iv) Finding a sparse subgraph,
v) Converting it to a min-cut max-flow problem by constructing an $s-t$ flow graph.

Given agents' utilities over arms and arms' utilities over agents, we first use agent-proposing DA algorithm to find the agent-optimal stable matching $m^{a}$ and use arm-proposing DA algorithm to find the arm-optimal stable matching $m^{b}$. 

For a given matching $m$ and an agent $a_{i}$, if $m(a_i) \ne m^b(a_i)$, we can have the \textbf{break-matching} operation: agent $a_i$ is now free, and arm $m(a_i)$ is semi-free, i.e., it only accepts a new proposal from an agent that it prefers to $m(a_i)$. 
The operation begins with agent $a_i$ proposing to the arm following $m(a_i)$ in the preference list, and this initiates a sequence of proposals, rejections, and acceptances given by the DA algorithm. 
It terminates when arm $m(a_i)$ accepts a new proposal.
By the break-matching operation on the matching $m$ and an agent $a^{(0)}$, we derive the rotation $\rho$ as a sequence of agent-arm pairs
\begin{equation*}
    \rho = (a^{(0)}, b^{(0)}), \ldots , (a^{(r-1)}, b^{(r-1)})
\end{equation*}
such that $m(a^{(i)}) = b^{(i)}$ for all $i$.
If each agent $a^{(i)}$ exchanges the partner $b^{(i)}$ for $b^{\textit{($i + 1$ mod $r$)}}$, then the new matching is also stable \citep{mcvitie1971stable}.
A rotation $\rho$ is said to be exposed to a matching $m$ if it can be derived from the break-matching operation on $m$.
We start with the agent-optimal stable matching $m^a$, and use break-matching operations to find all rotations until reaching $m^b$.

The next step is to construct a directed graph $G = (V,E)$, where the node set $V$ denotes all rotations.
A node $\rho$ has weight 
\begin{align*}
  w(\rho) &= \sum_{i = 0}^{r-1}(\mu_{a^{(i)}, b^{(i)}} -
  \mu_{a^{(i)}, b^{(\textit{$i+ 1$ mod $r$})}})\\
  &+ \sum_{j = 0}^{r-1}(\mu_{b^{(j)}, a^{(j)}} - \mu_{b^{(j)}, a^{(\textit{$j - 1$ mod $r$})}} ).
\end{align*}

By the construction, if $m'$ is the new matching after eliminating the rotation $\rho$ on the matching $m$, then we have 

\begin{equation*}
    R(m') = R(m) - w(\rho).
\end{equation*}
An edge $e \in E$ from rotation $\rho$ to $\pi$ denotes that $\rho$ is a predecessor of another rotation $\pi$, i.e., the rotation $\pi$ is exposed only after the rotation $\rho$ is eliminated.
The goal is to find the closed subset of the nodes that has the minimum weight sum.


The algorithm to find the minimum weight sum of a closed subsets \citep{irving1987efficient} is to find a sparse subgraph and convert it to a min-cut max-flow problem.
The minimum-weight closed subset can be derived in polynomial time of the market size. We refer the reader to the \Cref{sec:append_uti_opt} for further discussion on these techniques.

\begin{algorithm}[t]\small
    \SetKwInOut{Input}{Input}

    \Input{Time horizon $T$.}
    \For{each epoch $l = 1, 2, \ldots$}{ 
    \tcp{Exploration Phase:}
    Each agent pulls arms in a round-robin way for $N \cdot \ceil*{ log_{2}(l+1)}$ rounds;\\
    \tcp{Matching Phase:}
    Estimate $\hat{\mu}^a$ and $\hat{\mu}^b$;\\
    Apply \Cref{alg:weighted_egalitarian} based on $\hat{\mu}^a$ and $\hat{\mu}^b$ to find $\hat{m}$ 
    \\
    \tcp{Exploitation Phase:}
    Agent $a_i$ Pulls arm $\hat{m}(a_i)$ for $2^l$ turns for every $i$.
}
    \caption{Utilitarian Epoch ETC Algorithm }\label{alg:equalExplore}
\end{algorithm}

\subsection{Utilitarian Epoch ETC Algorithm} \label{sec:epochETC-util}
We are now ready to design an algorithm that finds a utilitarian-optimal stable matching when preferences must be learned.


Denote
\begin{equation}
    m_2 = {\arg\max}_{m \ne m^*, m \in \mathcal{M}} R(m)
\end{equation}
as a stable matching that has the second largest rewards besides the utilitarian-optimal stable matching $m^*$.
We assume such matching exists without loss of generality. 
Let
\begin{equation*}
    \delta = R(m^*) - R(m_2) 
\end{equation*}
be the utilitarian welfare difference between $m_2$ and the utilitarian-optimal stable matching $m^*$, which is useful for analysis.

We define regret by comparing the utility of the matched partner with the optimal partner and taking the sum over all agents and all arms at all times from $t = 1$ to $T$:

\begin{align*}
    Reg(T) = &\sum_{t = 1}^{T}\sum_{i = 1}^{N}\mathbb{E}[\mu_{a_i, m^*(a_i)} - \mu_{a_i, A_i(t)} ] \\
    &+ \sum_{t = 1}^{T}\sum_{j = 1}^{N}\mathbb{E}[\mu_{b_j, m^*(b_j)} - \mu_{b_j, A^{-1}_j(t)}],
\end{align*}
where $A^{-1}_j(t)$ denotes the agent that pulls arm $b_j$ at time $t$.

\paragraph{\textbf{Algorithm Description.}}
The proposed Epoch ETC algorithm (\Cref{alg:equalExplore}) combines the epoch-type uniform exploration and \Cref{alg:weighted_egalitarian}.
Each epoch is divided into three phases. 
The exploration phase runs in a round-robin fashion, i.e., each agent pulls an arm, and then the next agent pulls an arm from those not pulled before, and so on.
Formally, agent $a_i$ pulls arm $b_{(t+i) \mod N + 1}$ for time $t$.
This procedure ensures that there is no conflict in each round, and each agent pulls arms in a uniform way.
In epoch $l$, each agent pulls arms for $N \cdot \ceil*{log_2(l+1)}$ rounds, so each agent collects $\ceil*{log_2(l+1)}$ samples for each arm, and each arm collects $\ceil*{log_2(l+1)}$ samples for each agent.
%
Given the estimated utilities, \Cref{alg:weighted_egalitarian} is applied by committing to the matching for $2^{l}$ rounds.
%


\subsection{Analysis}
\label{eg_analysis}
In this section, we provide theoretical analysis on the regret bounds.
We first introduce the within-side minimum preference gap for analysis. 

\paragraph{\textbf{Within-side Minimum Preference Gap.}}

The \textit{agent within-side minimum preference gap} is defined as $\Delta^a = \min_{i}\min_{j \ne k}|\mu_{a_i,b_j} - \mu_{a_i, b_k}|$. Similarly, the \textit{arm within-side minimum preference gap} is defined as $\Delta^b = \min_{i}\min_{j \ne k}|\mu_{b_i,a_j} - \mu_{b_i,a_k}|$.

To analyze the regret bound of \Cref{alg:equalExplore}, we first provide a technical lemma that characterizes the condition for finding a \textit{stable} matching with respect to the aforementioned preference gaps.

\begin{lemma}\label{lem:stable} 
    Define a good event for agent $a_i$ and arm $b_j$ as
    $
        \mathcal{F}_{i,j} = \{|\mu_{a_i,b_j} - \hat{\mu}_{a_i,b_j}| \le \Delta^{a}/2 \} \cap \{|\mu_{b_j,a_j} - \hat{\mu}_{b_j,a_i}| \le \Delta^{b}/2 \} 
        \label{eq:fij}
    $, and define the intersection of the good events over all agents and all arms as
    $
        \mathcal{F} = \cap_{i \in [N], j \in [K]} \mathcal{F}_{i,j}
        $. 
        Then if the event $ \mathcal{F}$ occurs, the induced preference profile by the sample mean is the same with the true preference profile, i.e. $\hat{\mu}_{a_i,b_j} > \hat{\mu}_{a_i,b_k}$ if $\mu_{a_i,b_j} > \mu_{a_i, b_k}$, and $\hat{\mu}_{b_i,a_j} > \hat{\mu}_{b_i, a_k}$ if $\mu_{b_i, a_j} > \mu_{b_i, a_k}$ for all $i,j,k$.  
\end{lemma}

\begin{proof}
Assume that $\mu_{a_i, b_j} > \mu_{a_i, b_k}$, and more concretely, $\mu_{a_i,b_j} - \mu_{a_i,b_k} \geq \Delta^a$ by the definition of $\Delta^a$. We prove $\hat{\mu}_{a_i, b_j} > \hat{\mu}_{a_i, b_k}$ in the following. 
We have that
\begin{align*}
    \hat{\mu}_{a_i, b_j} &\geq \mu_{a_i, b_j} - \frac{\Delta^a}{2} & \text{[definition of }\mathcal{F}_{i,j}]\\
    &\geq \mu_{a_i, b_k} + \frac{\Delta^a}{2} & \text{[$\mu_{a_i, b_j} - \mu_{a_i, b_k} \geq \Delta^a$]}\\
    &\geq \hat{\mu}_{a_i, b_k}. & \text{[definition of }\mathcal{F}_{i,k}]
\end{align*}

Similarly, with the same proof logic, we have that $\hat{\mu}_{b_i, a_j} > \hat{\mu}_{b_i, a_k}$ if $\mu_{b_i, a_j} > \mu_{b_i, a_k}$.
\end{proof}


Unfortunately, small error in estimation could be fatal in finding the utilitarian-optimal matching. 
More specifically, \Cref{alg:weighted_egalitarian} may fail to find a utilitarian-optimal stable matching using estimated utilities even when the induced ordinal ranking is consistent with the true ordering. 
\Cref{ex:est_error} illustrates this point.
%

\begin{example} \label{ex:est_error}
    Consider two agents $\{a_1,a_2\}$ and two arms $\{ b_1,b_2\}$. 
    Assume true preferences are given below:

\begin{minipage}{0.45\linewidth}
\begin{align*}
a_1 &: \underline{b_2 (1.1)} \succ b_1 (0.4)^{*}  \\
a_2 &: \underline{b_1 (1.2)} \succ b_2 (0.6)^{*}
\end{align*}
\end{minipage}
\hfill
\begin{minipage}{0.45\linewidth}
\begin{align*}
b_1 &: a_1 (1.6)^{*} \succ \underline{a_2 (0.6)}  \\
b_2 &: a_2 (1.4)^{*} \succ \underline{a_1 (0.4)} 
\end{align*}
\end{minipage}

Both the underlined matching and the matching denoted by $^{*}$ are stable, and the matching denoted by $^*$ is the utilitarian-optimal stable matching.  
If the preferences are estimated as follows $a_1 : b_2 (1.82) \succ b_1 (0.3)$ and $a_2 : b_1 (1.79) \succ b_2 (0.5)$, then the underlined matching will be selected as the utilitarian-optimal stable matching, which is incorrect. 
Note that the estimated utilities are consistent with the ordering of the true preferences.
%
\end{example}

Given this observation, in the next lemma we characterize the condition in which some estimation error is tolerable by  \Cref{alg:weighted_egalitarian} when finding utilitarian-optimal stable matching.

\begin{lemma}
\label{lem:condition_correct}
If for all agent $a_i$ and arm $b_j$, the estimation can be bounded as 
\begin{equation*}
    |\hat{\mu}^{a}_{i,j} - \mu^{a}_{i,j}| < \min \{\frac{\delta}{4N}, \Delta^a/2\},
\end{equation*} 
and
\begin{equation*}
    |\hat{\mu}^{b}_{j,i} - \mu^{b}_{j,i}| < \min \{\frac{\delta}{4N}, \Delta^b/2\}
\end{equation*}
then we have that $\hat{m} = m^*$, where $\hat{m}$ is the matching computed by \Cref{alg:weighted_egalitarian} based on estimated utilities.
\end{lemma}
\begin{proof}
First, we know that $\hat{\mu}$ induces the true ranking by \Cref{lem:stable}, so $\hat{m}$ is guaranteed to be stable with respect to $\mu$.
For any matching $m$, notice that
\begin{equation}
    \label{eq:diff}
    |R(m) - \hat{R}(m)| \leq \delta/2
\end{equation}
by the following computation
\begin{align*}
    &|R(m) - \hat{R}(m)| \\
    &\leq |\sum_{i}\mu_{a_i, m(a_i)} - \hat{\mu}_{a_i, m(a_i)}| + |\sum_{j}\mu_{b_j, m(b_j)} - \hat{\mu}_{b_j, m(b_j)}| \\
    &\leq \sum_{i}|\mu_{a_i, m(a_i)} - \hat{\mu}_{a_i, m(a_i)}| + \sum_{j}|\mu_{b_j, m(b_j)} - \hat{\mu}_{b_j, m(b_j)}| \\
    &< 2N \cdot \frac{\delta}{4N} \\
    &= \delta/2.
\end{align*}
We then prove the claim by contradiction. If $\hat{m} \ne m^*$, since $\hat{m}$ is stable with respect to $\mu$, by the definition of $\delta$ we have
\begin{equation}
\label{eq:delta}
    R(m^*) \geq R(\hat{m}) + \delta.
\end{equation}
Thus, we have
\begin{align*}
    \hat{R}(m^*) &> R(m^*) - \frac{\delta}{2} & \text{[\Cref{eq:diff}]} \\
    &\geq R(\hat{m}) + \frac{\delta}{2} & \text{[\Cref{eq:delta}]}\\
    &> \hat{R}(\hat{m}). & \text{[\Cref{eq:diff}]}
\end{align*}
This is a contradiction to the definition of $\hat{m}$.
\end{proof}

The following technical lemma is useful to show the number of samples collected in the exploration phase. 

\begin{restatable}{lemma}{lemtechnical}
\label{lem:technical}
   \begin{equation*}
     \sum_{l' = 1}^{l}\ceil*{ log_2(l' + 1)} = (l+1)\ceil*{\log_2(l+1)} - 2^{\floor*{\log_2(l)} + 1} + 1.
\end{equation*} 
\end{restatable}

Next, we provide a lemma that (upper) bounds the probability that the computed matching does not coincide with the optimal matching $m^{*}$ in an epoch, when the epoch index is large enough.


\begin{lemma}[Error Probability]
\label{lem:error_prob}
    Denote $E_l$ as the event that in epoch $l$, the matching $\hat{m}$ returned by the \Cref{alg:weighted_egalitarian} is different from $m^*$. Then we can bound the error probability as 
    \begin{equation*}
    P(E_l) \leq 4N^2 \exp(-\frac{c_1 l \log(l)\beta^2}{2})
\end{equation*}
for any $l > l_0$, where $l_0$ and $c_1$ are two constants that are unrelated to the problem, and $\beta = \min\{\frac{\delta}{4N}, \Delta^a/2, \Delta^b/2\}$.
\end{lemma}

\begin{proof}
First, since the estimation of the expected rewards uses all the previous exploration phases, we compute the number of samples at epoch $l$ by \Cref{lem:technical}:
\begin{equation}
    \sum_{l' = 1}^{l}N \ceil*{ log(l' + 1)} = \Theta(N l \log(l)).
\end{equation}

In other words, there exists a constant $l_0$ and $c_1$, such that for any $l > l_0$, we have the lower bound on the number of samples $\sum_{l' = 1}^{l}N\ceil*{log(l' + 1)} \geq c_1 Nl \log(l)$. 

By \Cref{lem:condition_correct}, we have
\begin{align*}
   P(E_l) &\leq P[\exists i,j, s.t. |\mu_{a_i, b_j} - \hat{\mu}_{a_i, b_j}| \geq min\{\frac{\delta}{4N}, \Delta^a/2\} \\
   &\text{or} \ |\mu_{b_j, a_i} - \hat{\mu}_{b_j, a_i}| \geq min\{\frac{\delta}{4N}, \Delta^b/2\}] \\ 
   &\leq N^2 \cdot P[ |\mu_{a_i, b_j} - \hat{\mu}_{a_i, b_j}| \geq min\{\frac{\delta}{4N}, \Delta^a/2\}] \\
   &+ N^2 \cdot P[ |\mu_{b_j, a_i} - \hat{\mu}_{b_j, a_i}| \geq min\{\frac{\delta}{4N}, \Delta^b/2\}]
\end{align*}
by union bound.

Since $\mu_{a_i, b_j} - \hat{\mu}_{a_i, b_j}$ ($\mu_{b_j, a_i} - \hat{\mu}_{b_j, a_i}$) is the $\frac{1}{\sqrt{h}}$-subgaussian with mean 0 by \Cref{lem:subgaussian}, where $h = c_1 l \log(l)$ we have that 
\begin{align*}
    P[ |\mu_{a_i, b_j} - \hat{\mu}_{a_i, b_j}| &\geq min\{\frac{\delta}{4N}, \Delta^a/2\}] \leq 2 e^{-\frac{c_1 l \log(l)\beta^2}{2}} 
\end{align*}
and 
\begin{align*}
    P[ |\mu_{b_j, a_i} - \hat{\mu}_{b_j, a_i}| &\geq min\{\frac{\delta}{4N}, \Delta^b/2\}] \leq 2 e^{-\frac{c_1 l \log(l)\beta^2}{2}} 
\end{align*}
where $\beta = \min\{\frac{\delta}{4N}, \Delta^a/2, \Delta^b/2\}$.

Therefore, we conclude that
\begin{equation*}
    P(E_l) \leq 4N^2 \exp(-\frac{c_1 l \log(l)\beta^2}{2})
\end{equation*}
for any $l > l_0$.
\end{proof}




Bounding the error probability enables us to prove our main result in this section.

\begin{theorem}
\label{thm:ega_regret}
 Given a stable matching market problem with unknown utilities and $T$ planning horizon, \Cref{alg:equalExplore} computes the utilitarian-optimal stable matching with the regret of      \begin{equation*}
        Reg(T) = O(N^2 \log(T) \log(\log(T)) + N^3\log(T) + 2^{l_1}N),
    \end{equation*}
    where $l_1 = max \{ exp(\frac{2}{c_1 \beta^2}), l_0\}$, $\beta = min \{ \frac{\delta}{4N}, \Delta^a/2, \Delta^b/2\}$, and $c_1, l_0$ are two constants.

\end{theorem}

\begin{proof}
    Denote $R_l$ as the regret within epoch $l$, and denote $L$ as the number of epochs that starts within the $T$ turns.

    Since 
    \begin{align*}
        T &\geq \sum_{l = 1}^{L-1} 2^l \\
        &=2^L - 2,
    \end{align*}
    and so 
    \begin{equation}
        L \leq \log_{2}(T + 2)
        \label{eq:bound_ega_L}
    \end{equation}

        Then by \Cref{lem:error_prob}, since the worst regret for one time is $2N$, we compute the upper bound of expected regrets of epoch $l > l_0$:
\begin{align*}
    R_l &\leq 2N( N \ceil*{\log_2(l + 1)}  + P(E_l) \cdot 2^l) \\
    &\leq 2N( N \ceil*{\log_2(l + 1)} + 4N^2(2\gamma)^l),
\end{align*}
where $\gamma = \exp(- \frac{c_1 log(l)\beta^2}{2})$.

Furthermore, for any $l > l_1 = max(\exp(\frac{2}{c_1 \beta^2}), l_0)$, we have $\gamma < \frac{1}{2}$ and so $2\gamma < 1$.
Thus, for any $l > l_1$, we have  
\begin{equation*}
    R_l \leq 2N( N \ceil*{\log_2(l + 1)} + 4c_1N^2).
\end{equation*}

Finally we compute the cumulative regret
\begin{align*}
    Reg(T) &= \sum_{l = 1}^{L}R_l \\
    &\leq 2N(\sum_{l = 1}^{L} N \ceil*{\log_2(l + 1)} + \sum_{l = 1}^{l_1}2^l + \sum_{l = l_1 + 1}^{L} 4N^2) \\
    &\leq 2N( O(NL\log(L)) +  4N^2L + 2^{l_1 + 1}).
\end{align*}

Combining with \Cref{eq:bound_ega_L}, we have 
\begin{align*}
    &Reg(T) \\
    &\leq O(N^2 \log(T) \log(\log(T)) + N^3\log(T) + 2^{l_1}N)   
\end{align*}
and the proof is complete.
\end{proof}

\begin{remark}
    If we assume that $T$ is much larger than the instance-specific constants $N$ and $\beta$: $T \geq exp(exp(exp(\frac{1}{\beta^2})))$ and\\ $T \geq exp(exp(N))$, then we have 
    \begin{equation*}
        Reg(T) = O(N^2log(T) loglog(T)).
    \end{equation*}
\end{remark}

\begin{remark}\label{rem:unequal}
    When the number of agents is not equal to the number of arms, i.e. $N \ne K$, according to the rural hospital theorem~\citep{roth1986allocation}, there is a fixed set of agents and arms that are stable partners. Therefore, we can define the reward function and the regret function for agents and arms that are matched in a stable matching ($\min\{N,K\}$ agents and $\min\{N,K\}$ arms) with respect to true preference utilities. 
    The exploration phase lasts $max\{ N, K\} \ceil*{\log(l+1)}$ rounds in the $l$-th epoch so that all agents have pulled each arm for at least $\ceil*{\log(l+1)}$, and so do all arms.
    We redefine $\beta = min \{ \frac{\delta}{4 min\{N,K\}}, \Delta^a/2, \Delta^b / 2\}$, and revise the probability bound of \Cref{lem:error_prob} as \\$P(E_l) \leq 4\min \{ N, K\}^2 \exp(-\frac{c_1 l \log(l)\beta^2}{2})$, and the regret bound as $Reg(T) = O(\min \{N,K \} \max \{ N, K\} log(T) loglog(T))$.
    Similar analysis can be done for learning a maximin stable matching and we omit the details. 
\end{remark}
    
\section{Learning Maximin (Rawlsian) Stable Matching}

A Rawlsian approach to social welfare requires that the utility of the worst-off agent (irrespective of the side) to be maximized. 
In this section, we formally define maximin stable matchings, describe an algorithm to find such a matching when preferences are known, and design an algorithm based on epoch ETC with its regret bound.

\begin{algorithm}[t]\small
    \SetKwInOut{Input}{Input}
    \SetKwInOut{Output}{Output}
    \Input{Agents' utilities and arms' utilities.}
    \Output{The arm-side maximin stable matching} 
    Find the agent-optimal stable matching $m_0$, and find the arm-optimal stable matching;\\
    Let $b$ be an arm with minimum reward $R(m_i)$ in $m_i$, and let $a$ be its mate. If $a$ and $b$ are a pair in the arm-optimal stable matching, then stop and output $m_i$; $m_i$ is the arm-side maximin stable matching.
    Else perform break-matching operation on $m_i$ and $a$, and let $m_{i+1}$ be the resulting stable matching.\\
    If there are no arm with reward $R(m_{i+1})$ in $m_{i+1}$, then stop and output $m_i$; $m_i$ is the arm-side maximin stable matching. Else set $i = i+1$ and go to previous step.

    \caption{Arm-side Maximin (Rawlsian) Algorithm}\label{alg:weighted_maximin}
\end{algorithm}

\begin{algorithm}[t]\small
    \SetKwInOut{Input}{Input}

    \Input{Time horizon $T$.}
    \For{each epoch $l = 1, 2, \ldots$}{ 
    \tcp{Exploration Phase:}
    Each agent pulls arms in a round-robin way for $N \cdot \ceil*{ log_{2}(l+1)}$ rounds;\\
    \tcp{Matching Phase:}
    Estimate $\hat{\mu}^a$ and $\hat{\mu}^b$;\\
    Apply \Cref{alg:weighted_maximin} based on $\hat{\mu}^a$ and $\hat{\mu}^b$ to find $\hat{m}$; 
    \tcp{Exploitation Phase:}
    Agent $a_i$ Pulls arm $\hat{m}(a_i)$ for $2^l$ turns for every $i$.
}
    \caption{Maximin Epoch ETC Algorithm}\label{alg:equalExplore-minimax}
\end{algorithm}

\subsection{Maximin Stable Matching}

The maximin stable matching is a stable matching where the minimum reward across all agents and arms is maximized.

\begin{definition}[Maximin Stable Matching]
    Given a matching, a minimum reward on both side of the markets is \begin{equation*}
    R(m) = \min  \{ \{ \mu_{a_i,m(a_i)}\}_{i} \cup \{ \mu_{j, m(b_j)}\}_{j} \},
\end{equation*}
    
    The \textbf{maximin stable matching}, $m^{*}\in \mathcal{M}$, is a stable matching that maximize the minimum reward, i.e. $m^* = \arg\max_{m\in \mathcal{M}} R(m)$.
\end{definition}

Similarly, we have the estimated minimum rewards across agents and arms
\begin{equation*}
    \hat{R}(m) = \min \{  \{ \hat{\mu}_{a_i,m(a_i)}\}_{i} \cup  \{ \hat{\mu}_{b_j, m(b_j)}\}_{j} \},
\end{equation*}
and $\hat{m}$ as the stable matching with respect to $\hat{\mu}$ that has the largest estimated mimimum rewards.

The maximin stable matching algorithm~\citep{gusfield1987three} can be used to compute the maximin stable matching when preferences are known.
Although the original paper studied ordinal preference, we slightly change the algorithm for cardinal utilities and state it in \Cref{alg:weighted_maximin}.

\paragraph{\textbf{Computing a Maximin Stable Matching.}}
Given agents' and arms' utilities over the other side, the goal is to find the maximin stable matching.
The problem can be divided into two subproblems: finding the arm-side maximin stable matching and the agent-side maximin stable matching.
For the first subproblem, we only consider the stable matchings such that one arm has the minimum reward, and the arm-side maximin stable matching is the matching that maximizes the minimum reward among these matchings.
Similarly, if we only consider the stable matchings such that one agent has the minimum reward, and the agent-side maximin stable matching is the one that maximizes the minimum reward among these stable matchings.
\Cref{alg:weighted_maximin} shows the procedure to find the arm-side maximin stable matching.
The algorithm to find the agent-side maximin stable matching can be derived by switching the roles of agents and arms.
The algorithm starts from the agent-optimal stable matching, and breaks the matching for the agent-arm pair in which the arm has the minimum reward.
The algorithm stops either when one agent is matched to the partner in the arm-optimal stable matching, or when no arm has the minimum reward.

\subsection{Maximin Epoch ETC Algorithm}
In this section, we propose an algorithm that has similar structure to \Cref{alg:equalExplore} to find the maximin stable matching when preferences are unknown.

We define the maximin regret by comparing the minimum reward of the matching with the optimal minimum reward, and taking sum for all time:
\begin{equation*}
   Reg^{MM}(T) = \sum_{t = 1}^{T}\mathbb{E}[R(m^*) - \min \{  \{ \mu_{a_i, A_i(t)}\}_{i} \cup \{ \mu_{b_j, A^{-1}_j(t)}\}_{j}\} ].
\end{equation*}

The proposed \Cref{alg:equalExplore-minimax} combines the epoch-type ETC with \Cref{alg:weighted_maximin}. 
In each epoch $l$, each agent pulls arms for $N \cdot \ceil*{\log_2(l+1)}$ rounds in a round-robin way.
The difference between \Cref{alg:equalExplore-minimax} and \Cref{alg:equalExplore} is the matching phase.
Given the sample average, agents commit to the matching computed by \Cref{alg:weighted_maximin} for $2^l$ rounds.

\subsection{Analysis}

In this section, we provide theoretical results on the regret bounds. 
We introduce the cross-side minimum preference gap that minimizes the preference gap across all agents or arms, while the within-side minimum preference gap defined in \Cref{eg_analysis} is to minimize the preference gap within individuals.

\paragraph{\textbf{ Cross-side Minimum Preference Gap.}}

We define the agent cross-side minimum preference gap as\\
$\Gamma^{a} =  \min_{\substack{i,j,k,l \\ \mu_{a_i, b_j} - \mu_{a_k, b_l} \neq 0}} \left| \mu_{a_i, b_j} - \mu_{a_k, b_l} \right|$, and the arm cross-side minimum preference gap as
$\Gamma^{b} =  \min_{\substack{i,j,k,l \\ \mu_{b_i, a_j} - \mu_{b_k, a_l} \neq 0}} \left| \mu_{b_i, a_j} - \mu_{b_k, a_l} \right| $, and $\Gamma = \min \{ \Gamma^a, \Gamma^b\}$.
By definition, we have $\Gamma > 0$ and $\Gamma^{a} \leq \Delta^a, \Gamma^{b} \leq \Delta^b$.


The following lemma provides a condition when some estimation error is tolerable to find maximin stable matching.

\begin{lemma}
\label{lem:minimax_condition_correct}
If for all agent $a_i$ and arm $b_j$, the estimation can be bounded as 
\begin{equation*}
    |\hat{\mu}_{a_i, b_j} - \mu_{a_i, b_j}| <  \Gamma/2,
\end{equation*} 
and
\begin{equation*}
    |\hat{\mu}^{b}_{j,i} - \mu^{b}_{j,i}| <  \Gamma/2,
\end{equation*}
then we have that $R(\hat{m}) = \max_{m \in \mathcal{M}} R(m)$, where $\hat{m}$ is the matching computed by \Cref{alg:equalExplore-minimax} based on estimations $\hat{\mu}$.
\end{lemma}

\begin{proof}
    By \Cref{lem:stable}, $\hat{\mu}$ induces the true ordering within all agents/arms and so the matching $\hat{m}$ is guaranteed to be stable with respect to $\mu$.
    In other words, a matching is stable with respect to $\hat{\mu}$ if and only if it is stable with respect to $\mu$.

    We then prove the lemma by contradiction. Assume that $R(\hat{m}) < R(m^*)$, and by the construction of $\Gamma$, $R(\hat{m}) + \Gamma \leq R(m^*)$.
    Since the estimated utility does not deviate from the true value more than $\Gamma/2$, we have that for any matching $m$,
    \begin{equation*}
       |\hat{R}(m) - R(m)| < \Gamma/2. 
    \end{equation*} 
    Therefore, we have
    \begin{equation*}
       \hat{R}(m^*) > R(m^*) - \Gamma/2, 
    \end{equation*} 
    and
    \begin{equation*}
       \hat{R}(\hat{m}) < R(\hat{m}) + \Gamma/2.  
    \end{equation*}
    These equations combined with $R(\hat{m}) + \Gamma \leq R(m^*)$ immediately gets
    \begin{equation*}
      \hat{R}(m^*) > R(m^*) - \Gamma/2 \geq  R(\hat{m}) + \Gamma/2 > \hat{R}(\hat{m}). 
    \end{equation*}
    This is a contradiction since $m^*$ is stable with respect to $\hat{\mu}$ and should have been chosen as the maximin stable matching with respect to the estimated utility profile, $\hat{\mu}$.
    \end{proof}

    Next, we provide a lemma that bounds the probability that the matching found in an epoch does not coincide with the maximin stable matching.

\begin{lemma}[Error Probability]
\label{lem:minimax_error_prob}
    Denote $E_l$ as the event that in epoch $l$, the matching $\hat{m}$ returned in the matching phase is different from $m^*$. Then we can bound the error probability as 
    \begin{equation*}
    P(E_l) \leq 4N^2 \exp(-\frac{c_1 l \log(l)\Gamma^2}{8})
\end{equation*}
for any $l > l_0$, where $l_0$ and $c_1$ are two constants.
\end{lemma}

\begin{proof}
First, since the estimation of the expected rewards uses all the previous exploration phases, we compute the number of samples at epoch $l$ by \Cref{lem:technical}:
\begin{equation}
    \sum_{l' = 1}^{l}N \ceil*{ log(l' + 1)} = \Theta(N l \log(l)).
\end{equation}

In other words, there exists a constant $l_0$ and $c_1$, such that for any $l > l_0$, we have the lower bound on the number of samples $\sum_{l' = 1}^{l}N\ceil*{log(l' + 1)} \geq c_1 Nl \log(l)$. 

By \Cref{lem:minimax_condition_correct}, we have
\begin{align*}
   P(E_l) &\leq P[\exists i,j, s.t. |\mu_{a_i, b_j} - \hat{\mu}_{a_i, b_j}| \geq \Gamma/2 \ \text{or} \ |\mu_{b_j, a_i} - \hat{\mu}_{b_j, a_i}| \geq  \Gamma/2] \\ 
   &\leq N^2 \cdot P[ |\mu_{a_i, b_j} - \hat{\mu}_{a_i, b_j}| \geq  \Gamma/2] \\
   &+ N^2 \cdot P[ |\mu_{b_j, a_i} - \hat{\mu}_{b_j, a_i}| \geq  \Gamma/2].
\end{align*}
Since $\mu_{a_i, b_j} - \hat{\mu}_{a_i, b_j}$ ($\mu_{b_j, a_i} - \hat{\mu}_{b_j, a_i}$) is the $\frac{1}{\sqrt{h}}$-subgaussian with mean 0 by \Cref{lem:subgaussian}, where $h = c_1 l \log(l)$ we have that 
\begin{align*}
    P[ |\mu_{a_i, b_j} - \hat{\mu}_{a_i, b_j}| &\geq  \Gamma/2] \leq 2 e^{-\frac{c_1 l \log(l)\Gamma^2}{8}} & \text{[Definition of subgaussian]}
\end{align*}
and 
\begin{align*}
    P[ |\mu_{b_j, a_i} - \hat{\mu}_{b_j, a_i}| &\geq  \Gamma/2] \leq 2 e^{-\frac{c_1 l \log(l)\Gamma^2}{8}} & \text{[Definition of subgaussian]}.
\end{align*}

Therefore, we conclude that
\begin{equation*}
    P(E_l) \leq 4N^2 \exp(-\frac{c_1 l \log(l)\Gamma^2}{8})
\end{equation*}
for any $l > l_0$.
\end{proof}

The main theorem shows the regret bound for \Cref{alg:equalExplore-minimax} with the error probability lemma.

\begin{restatable}{theorem}{thmmmregret}

\label{thm:mm_regret}
Given a stable matching market problem with unknown utilities and $T$ planning horizon, \Cref{alg:equalExplore-minimax} computes the maximin stable matching with regret    \begin{equation*} 
    Reg^{MM} (T) = O(N \log(T) \log(\log(T)) +  4N^2\log(T) + 2^{l_1}),
    \end{equation*}
    where $l_1 = max \{ exp(\frac{8}{c_1\Gamma^2}), l_0\}$, and $l_0, c_1$ are two constants.
\end{restatable}

The detailed proof is in \Cref{append:proof_thm4}.

\begin{remark}
    If we assume that $T$ is much larger than the instance-specific constants $N$ and $\Gamma$: $T \geq exp(exp(exp(\frac{1}{\Gamma^2})))$ and\\ $T \geq exp(exp(N))$, then we have 
    \begin{equation*}
        Reg(T) = O(Nlog(T) loglog(T)).
    \end{equation*}
\end{remark}

\begin{figure*}[t]
    \centering
    \begin{subfigure}{0.28\textwidth} 
        \centering
        \includegraphics[width=\linewidth,height=1.2in]{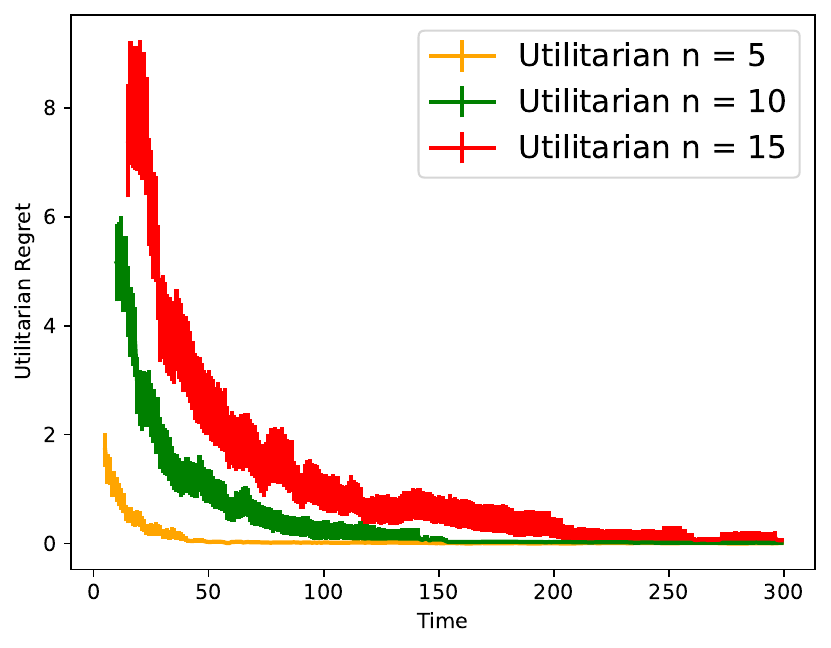} 
    \end{subfigure}
    \hfill 
    \begin{subfigure}{0.28\textwidth} 
        \centering
        \includegraphics[width=\linewidth,height=1.2in]{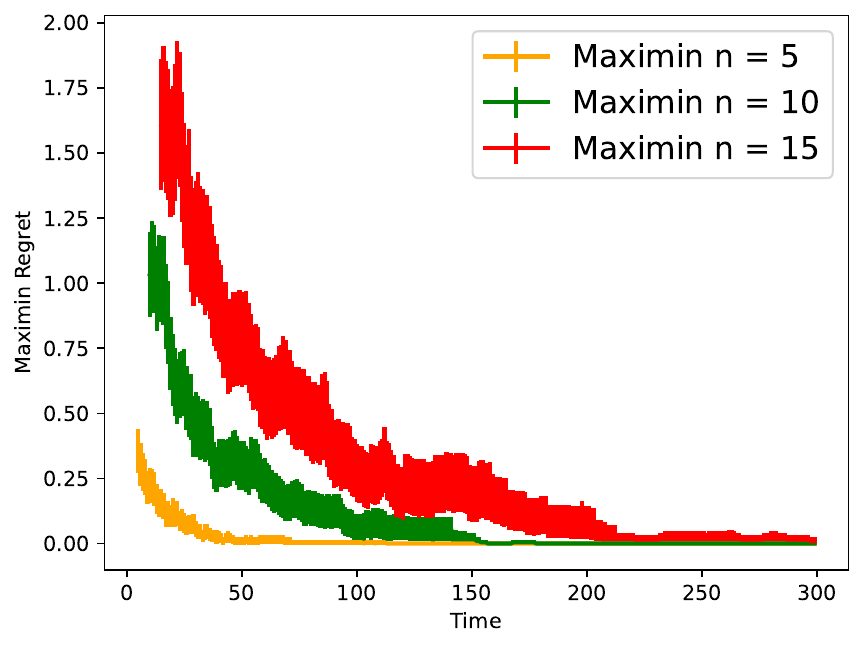} 
    \end{subfigure}
    \hfill 
    \begin{subfigure}{0.28\textwidth} 
        \centering
        \includegraphics[width=\linewidth,height=1.2in]{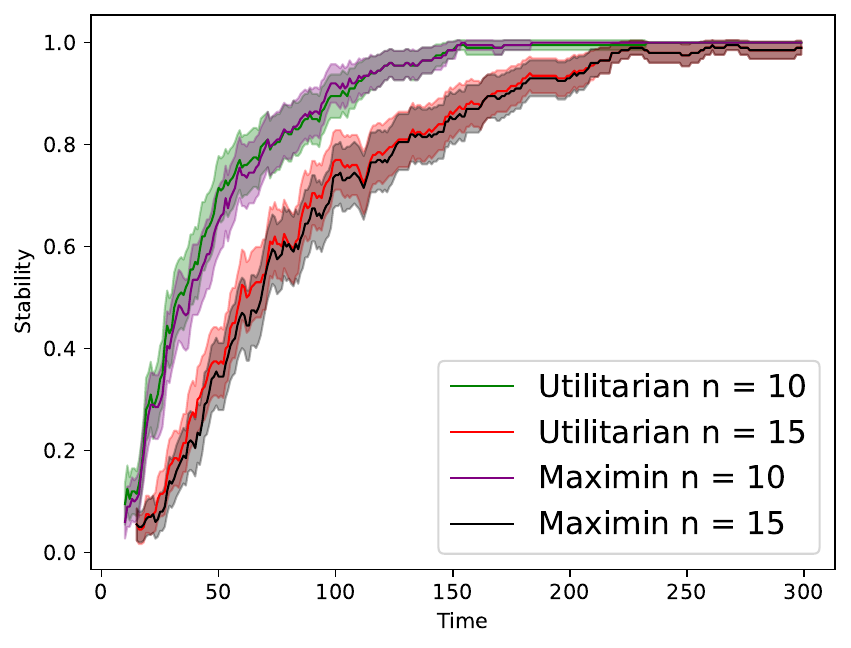} 
    \end{subfigure}
    \caption{$95 \%$ confidence interval of utilitarian regrets, maximin regrets, and stability; 200 preference profiles generated uniformly at random.}
    \label{fig.pdf}
\end{figure*}
\section{Experimental Results}
In this section, we experimentally validate our theoretical results by examining the regrets and stability for two types of Epoch ETC algorithms.
For this, we consider $N = K = 5, 10, 15$ and randomly generate preferences.
In particular, we follow similar constructions with previous literature \cite{liu2021bandit, hosseini2024puttinggaleshapley}: for each agent $a_i$, the true utilities $\{ \mu_{a_i,b_1},\mu_{a_i,b_2}, \ldots, \mu_{a_i,b_N} \}$ are randomized permutations of the sequence $\{ 1, 2,\ldots, N \}$.  
Arms' preferences are generated the same way independently.
We conduct 200 independent simulations, with each simulation featuring a randomized true preference profile.
We compute the utilitarian-optimal stable matching by \Cref{alg:equalExplore} and the maximin-optimal stable matching by \Cref{alg:equalExplore-minimax}, and examine utilitarian regrets (maximin regrets) as well as stability.

The first two subfigures in \Cref{fig.pdf} show the utilitarian regret for \Cref{alg:equalExplore} and the maximin regret for \Cref{alg:equalExplore-minimax}, where time is the total number of samples collected for the exploration phase.
For both algorithms, utilitarian regret and maximin regret converge to $0$ as the number of explorations increases. 
It also shows that the number of samples to find the utilitarian-optimal stable matching (or maximin stable matching) increases when the number of agents/arms increases.
The last subfigure shows the average stability for two algorithms.
We note that from the figure, \Cref{alg:equalExplore} and \Cref{alg:equalExplore-minimax} have similar performances in terms of stability.

The following two examples show that when estimation errors occur, stability of one solution does not ensure the stability of another. 
\Cref{ex:1} demonstrates that \Cref{alg:equalExplore} does not reach a stable matching when there exists some estimation error, while \Cref{alg:equalExplore-minimax} reaches a stable matching with the same estimation; \Cref{ex:2} shows the vice versa:  \Cref{alg:equalExplore-minimax} does not reach a stable matching while \Cref{alg:equalExplore} does.

\begin{example}
\label{ex:1}
    Consider two agents $\{a_1,a_2\}$ and two arms $\{ b_1,b_2\}$. 
    Assume true preferences are given below:

\begin{minipage}{0.45\linewidth}
\begin{align*}
a_1 &: \underline{b_2 (1)} \succ b_1 (0)^*  \\
a_2 &: b_2 (1)^* \succ \underline{b_1 (0)}
\end{align*}
\end{minipage}
\hfill
\begin{minipage}{0.45\linewidth}
\begin{align*}
b_1 &: a_1 (1)^{*} \succ \underline{a_2 (0)}  \\
b_2 &: a_2 (1)^* \succ \underline{a_1 (0)} 
\end{align*}
\end{minipage}

The matching denoted by $^*$ is the only stable matching.
When estimated preferences for agents are:
\begin{equation*}
   a_1 : b_2 (1.03) \succ b_1 (0.44), 
a_2 : b_1 (1.79) \succ b_2 (0.13) ,
\end{equation*}
then the maximin stable matching with respect to the estimation is stable, but utilitarian-optimal stable matching with respect to the estimation is the underlined matching and unstable.
\end{example}

\begin{example}
\label{ex:2}
    Consider two agents $\{a_1,a_2\}$ and two arms $\{ b_1,b_2\}$. 
    Assume true preferences are given below:

\begin{minipage}{0.45\linewidth}
\begin{align*}
a_1 &: b_1 (1)^* \succ \underline{b_2 (0.5)}  \\
a_2 &: \underline{b_1 (1)} \succ b_2 (0.5) ^*
\end{align*}
\end{minipage}
\hfill
\begin{minipage}{0.45\linewidth}
\begin{align*}
b_1 &: a_1 (1) ^* \succ \underline{a_2 (0.5)}  \\
b_2 &: a_2 (1) ^* \succ \underline{a_1 (0.5)} 
\end{align*}
\end{minipage}

The matching denoted by $^*$ is the only stable matching.
When estimated preferences for agents are:
\begin{equation*}
   a_1 : b_2 (0.69) \succ b_1 (0.02), 
a_2 : b_1 (1.68) \succ b_2 (1.38) ,
\end{equation*}
then the utilitarian-optimal stable matching with respect to the estimation is stable, but maximin stable matching with respect to the estimation is the underlined matching and unstable.

\end{example}


\section{Concluding Remarks}
In this paper, we study the bandit learning problem in two-sided matching markets, when both sides of the market are unaware of their utilities and must learn through sampling.
We focus on two welfarist approaches: Utilitarian and Rawlsian.
We proposed two types of epoch ETC algorithms and analyze their regret bounds.
The analysis is based on two different types of minimum preference gaps: within-side and cross-side.

We conclude by discussing some future directions.
First, we can consider extending the problem to many-to-one matching markets.
Studying the sample complexity to reach a utilitarian-optiaml stable matching (or maximin stable matching) under the PAC setting is another future direction.
Lastly, we can consider a more general setting where preferences are not strict and can have ties.

\bibliographystyle{plainnat}
\bibliography{reference}

\appendix

\section{Additional Proofs and Technical Lemmas}
\begin{lemma}[Property of independent subgaussian, Lemma 5.4 in \cite{lattimore2020bandit}]\label{lem:subgaussian}
Suppose that $X$ is $d$-subgaussian and $X_1$ and $X_2$ are independent and $d_1$ and $d_2$ subgaussian, respectively, then we have the following property:\\
    (1) $Var[X] \le d^2$. \\
    (2) $cX$ is |c|d-subgaussian for all $c \in \mathbbm{R}$. \\
    (3) $X_1 + X_2$ is $\sqrt{d_1^2 + d_2^2}$-subgaussian.
\end{lemma}

\subsection{The Proof of \Cref{lem:technical}}

\lemtechnical*

\begin{proof}
    First we have the decomposition
    \begin{align}
        \sum_{l' = 1}^{l}\ceil*{ log_2(l' + 1)}  &= (l+1)\ceil*{\log_2(l+1)} \\
        &- \sum_{l' = 1}^{l}l'(\ceil*{\log_2(l'+1)} - \ceil*{\log_2(l')}),
    \end{align}
      and
    \begin{equation*}
      \ceil*{\log_2(l' + 1)} - \ceil*{\log_2(l')} = 1  
    \end{equation*}
    if $l'$ is a power of 2, and $0$ otherwise. 
    Then we simplify the second term of the decompostion 
    \begin{align*}
        \sum_{l' = 1}^{l}l'(\ceil*{\log_2(l'+1)} - \ceil*{\log_2(l')}) &= \sum_{l' = 1}^{l}l' \cdot \mathbf{1}\{\textit{$l'$ is a power of $2$}\} \\
        &= \sum_{t  = 0}^{\floor*{\log_2(l)}}2^t \\
        &=2^{\floor*{\log_2(l)} + 1} - 1.
    \end{align*}
    Substituting it back to the decomposition gets the result.
\end{proof}

\subsection{The proof of \Cref{thm:mm_regret} }
\label{append:proof_thm4}
\thmmmregret*

\begin{proof}
    Denote $R_l$ as the regret within epoch $l$, and denote $L$ as the number of epochs that starts within the $T$ turns.

    Since 
    \begin{align*}
        T &\geq \sum_{l = 1}^{L-1} 2^l \\
        &=2^L - 2,
    \end{align*}
    and so 
    \begin{equation}
        L \leq \log_{2}(T + 2)
        \label{eq:bound_L}
    \end{equation}

        Then by \Cref{lem:minimax_error_prob}, since the worst case regret is a constant at each time, we compute the upper bound of expected regrets of epoch $l > l_0$:
\begin{align*}
    R_l &\leq  N \ceil*{\log_2(l + 1)} + P(E_l) \cdot 2^l  \\
    &\leq  N \ceil*{\log_2(l + 1)}  + 4N^2(2\gamma)^l,
\end{align*}
where $\gamma = \exp(- \frac{c_1 log(l)\Gamma^2}{8})$.

Furthermore, for any $l > l_1 = max(\exp(\frac{8}{c_1 \Gamma^2}), l_0)$, we have $\gamma < \frac{1}{2}$ and so $2\gamma < 1$.
Thus, for any $l > l_1$, we have  
\begin{equation*}
    R_l \leq  N \ceil*{\log_2(l + 1)} + 4N^2.
\end{equation*}

Finally we compute the cumulative regret
\begin{align*}
    Reg^{MM}(T) &= \sum_{l = 1}^{L}R_l \\
    &\leq \sum_{l = 1}^{L} N \ceil*{\log_2(l + 1)} + \sum_{l = 1}^{l_1}2^l + \sum_{l = l_1 + 1}^{L} 4N^2\\
    &\leq  O(NL\log(L)) + 4N^2L + 2^{l_1 + 1}.
\end{align*}

Combining with \Cref{eq:bound_L}, we have 
\begin{equation*}
    Reg^{MM}(T) = O(N \log(T) \log(\log(T)) +  4N^2\log(T) + 2^{l_1})    
\end{equation*}
and the proof is complete.
\end{proof}

\section{Computing a Utilitarian-Optimal Stable Matching}
\label{sec:append_uti_opt}

In this section, we introduce the techniques that are omitted in Section 3. We introduce how to construct a sparse subgraph such that the subgraph keeps the closed subset.
Then we convert the problem that finds minimum-weight closed subset in the subgraph to finding a minimum cut in an $s-t$ flow graph.

\paragraph{\textbf{Constructing the sparse subgraph.}}
When we have the original graph where nodes denote the rotations, edges denote predecessor relations, and each node has an associated weight, we need to reconstruct a sparse subgraph where the node sets are the same, but partial original edges are kept.
%
Edges defined by two rules are added to the subgraph: (i) If $(a_i, b_j)$ is a member of a rotation $\rho_1$, and $b_k$ is the first arm below $b_j$ in $a_i$'s list such that $(a_i, b_k)$ is a member of another rotation $\rho_2$, then we have an edge from $\rho_1$ to $\rho_2$; (ii) If $(a_i, b_k)$ is not a member of any rotation, but is eliminated by a rotation $\rho_2$, and $b_j$ is the first arm above $b_k$ in $a_i$'s list such that $(a_i, b_j)$ is a member of another rotation $\rho_1$, then we have an edge from $\rho_1$ to $\rho_2$.
The key observation is that by such construction, the number of edges is bounded by $O(N^2)$, and the subgraph keeps the closed subsets of the original graph, i.e. a subset of nodes is closed in the original graph if and only if it is also closed in the subgraph.  

\paragraph{\textbf{Solving the induced $s-t$ maximim flow problem.}}
After getting the sparse subgraph, we can convert the original problem to solving a min-cut max-flow problem.
We add source node $s$, sink node $t$, and all rotation nodes to the $s-t$ flow graph.
A directed edge is added from $s$ to every positive node (node that has positive weight $w(\rho) > 0$); the capacity of the edge is $w(\rho)$.
A directed edge is added from a negative node (node that has negative weight $w(\rho) < 0$) to $t$; the capacity of the edge is $|w(\rho)|$.
The capacity of the original edge in the subgraph is set to be infinity.
The negative nodes whose edges into $t$ are uncut by the minimum cut in $s-t$ and their predecessors are the rotations that define the stable utilitarian-optimal match.

\section{Sample Complexity}

In this section, we introduce the sample complexity of finding an optimal stable solution (both utilitarian optimal and maximin) under the probably approximately correct (PAC) framework.
Firstly, we study the ETC algorithm combined with \Cref{alg:weighted_egalitarian}.
We study how many samples $T$ are needed to find the utilitarian-optimal stable matching given a fixed probability budget.

\begin{theorem}
    With probability at least $1 - \alpha$, the ETC with Utilitarian-Optimal algorithm finds the utilitarian-optimal stable matching with sample complexity $\frac{2N^2}{\beta^2}\log(\frac{4N^2}{\alpha})$.
\end{theorem}

\begin{proof}
Since agents pull arms uniformly, we assume each agent samples each arm $t$ times.
By \Cref{lem:subgaussian} and the definition of subgaussian, with probability at least $1 - 2 exp(-\frac{\beta^2 t}{2})$, we have that
\begin{equation*}
    |\hat{\mu}^a_{i,j} - \mu^a_{i,j}| \leq \beta.
\end{equation*}
Then by a union bound, with probability at least $1 - 2N^2exp(-\frac{\beta^2 t}{2})$, we have $|\hat{\mu}^a_{i,j} - \mu^a_{i,j}| \leq \beta$ for any pair of $i,j$.
In symmetry, since $N=K$ at the same time we also have $|\hat{\mu}^b_{j,i} - \mu^b_{j,i}| \leq \beta$ for any pair of $j,i$.
Therefore, by \Cref{lem:condition_correct}, with probability at least $1 - 4N^2exp(-\frac{\beta^2 t}{2})$, we have that $\hat{m} = m^*$, where $m^*$ is the utilitarian-optimal stable matching.  
By setting $\alpha = 4N^2exp(-\frac{\beta^2 t}{2})$, we have $t = \frac{2}{\beta^2}\log(\frac{4N^2}{\alpha})$.
Therefore, with probability at least $1 - \alpha$, the ETC algorithm needs the total samples of the size $T = \frac{2N^2}{\beta^2}\log(\frac{4N^2}{\alpha})$ to find the utilitarian-optimal stable matching.
\end{proof}

Similarly, we have the following result for finding a maximin stable matching with ETC algorithm.

\begin{theorem}
    With probability at least $1 - \alpha$, the ETC with maximin optimal algorithm finds the maximin stable matching with sample complexity $\frac{8N^2}{\Gamma^2}\log(\frac{4N^2}{\alpha})$.
\end{theorem}

\begin{proof}
Assume each agent samples each arm $t$ times.
By \Cref{lem:subgaussian} and the definition of subgaussian, with probability at least $1 - 2 exp(-\frac{\Gamma^2 t}{8})$, we have that
\begin{equation*}
    |\hat{\mu}^a_{i,j} - \mu^a_{i,j}| \leq \Gamma/2.
\end{equation*}
Then by a union bound, with probability at least $1 - 2N^2exp(-\frac{\Gamma^2 t}{8})$, we have $|\hat{\mu}^a_{i,j} - \mu^a_{i,j}| \leq \Gamma/2$ for any pair of $i,j$.
We also have $|\hat{\mu}^b_{j,i} - \mu^b_{j,i}| \leq \Gamma/2$ for any pair of $j,i$.
Therefore, by \Cref{lem:minimax_condition_correct}, with probability at least $1 - 4N^2exp(-\frac{\Gamma^2 t}{8})$, we have that $\hat{m} = m^*$, where $m^*$ is the utilitarian-optimal stable matching.  
By setting $\alpha = 4N^2exp(-\frac{\Gamma^2 t}{8})$, we have $t = \frac{8}{\Gamma^2}\log(\frac{4N^2}{\alpha})$.
Therefore, with probability at least $1 - \alpha$, the ETC algorithm needs the total samples of the size $T = \frac{8N^2}{\Gamma^2}\log(\frac{4N^2}{\alpha})$ to find the utilitarian-optimal stable matching.
\end{proof}

\end{document}